\documentclass[10pt,twocolumn,letterpaper]{article}

\usepackage{cvpr}      % To produce the REVIEW version

\def\cvprPaperID{11698} % *** Enter the CVPR Paper ID here
\def\confName{CVPR}
\def\confYear{2022}

\usepackage{graphicx}
\usepackage{amsmath}
\usepackage{amssymb}
\usepackage{booktabs}

\makeatletter
\@namedef{ver@everyshi.sty}{}
\makeatother
\usepackage{tikz}

\relax
\usepackage{times}  % DO NOT CHANGE THIS
\usepackage{helvet} % DO NOT CHANGE THIS
\usepackage{courier}  % DO NOT CHANGE THIS
\usepackage{graphicx} % DO NOT CHANGE THIS
\usepackage{caption} % DO NOT CHANGE THIS AND DO NOT ADD ANY OPTIONS TO IT
\usepackage{xspace}
\usepackage{amsthm}
\usepackage{booktabs}

\usepackage{multirow}
\usepackage{ upgreek }
\usepackage[accsupp]{axessibility}

\usepackage{wrapfig}
\usepackage[export]{adjustbox}
\usepackage{algorithm}
\usepackage{algorithmic}
\usepackage{tcolorbox}

\newtheorem{theorem}{Theorem}
\newtheorem{lemma}[theorem]{Lemma}

\usepackage{graphicx}

\usepackage{amsmath}

\newcommand\figref[1]{Fig.~\ref{#1}}

\newcommand\tabref[1]{Table~\ref{#1}}

\newcommand\secref[1]{Sec.~\ref{#1}}
\newcommand\equref[1]{Eq.(\ref{#1})}

\usepackage{framed}

\usepackage{appendix}

\usepackage[pagebackref,breaklinks,colorlinks]{hyperref}

\usepackage[capitalize]{cleveref}
\crefname{section}{Sec.}{Secs.}
\Crefname{section}{Section}{Sections}
\Crefname{table}{Table}{Tables}
\crefname{table}{Tab.}{Tabs.}
\newcommand{\tool}{NICGSlowDown\xspace}

\def\cvprPaperID{11698} % *** Enter the CVPR Paper ID here
\def\confName{CVPR}
\def\confYear{2022}

\begin{document}
\title{\tool: Evaluating the Efficiency Robustness of \\ Neural Image Caption Generation Models}

\author{Simin Chen \xspace\xspace\xspace\xspace\xspace\xspace\xspace Zihe Song \xspace\xspace\xspace\xspace\xspace\xspace\xspace   Mirazul Haque \xspace\xspace\xspace\xspace\xspace\xspace\xspace Cong Liu   \xspace\xspace\xspace\xspace\xspace\xspace\xspace Wei Yang\\
The University of Texas at Dallas\\
800 W Campbell Rd, Richardson, TX 75080\\
{\tt\small \{simin.chen, zihe.song, mirazul.haque, cong, wei.yang\}@utdallas.edu}
% \and
}
\maketitle

\pagestyle{empty} 

\begin{abstract}
Neural image caption generation~(NICG) models have received massive attention from the research community due to their excellent performance in visual understanding. Existing work focuses on improving NICG model accuracy while efficiency is less explored. However, many real-world applications require real-time feedback, which highly relies on the efficiency of NICG models. Recent research observed that the efficiency of NICG models could vary for different inputs. This observation brings in a new attack surface of NICG models, i.e., An adversary might be able to slightly change inputs to cause the NICG models to consume more computational resources. To further understand such efficiency-oriented threats,  we propose a new attack approach, \tool, to evaluate the efficiency robustness of NICG models. Our experimental results show that \tool can generate images with human-unnoticeable perturbations that will increase the NICG model latency up to 483.86\%. We hope this research could raise the community’s concern about the efficiency robustness of NICG models.
\end{abstract}

\section{Introduction}
Neural Image Caption Generation~(NICG) models have received wide attention from both academia and industry in recent years~\cite{anderson18cvpr,gan17cvpr,vaswani17nips,vinyals15cvpr,wang20mm}. 
NICG model combines computer vision and natural language processing techniques for image understanding and textual description generation. 
Designing NICG models is a challenging task but could have a massive impact in the real world~\cite{anderson18cvpr,cornia20cvpr,huang19iccv,pan20cvpr}, such as helping people with visual impairment to understand visual inputs, enhancing the accuracy of image search engines, or transferring images to text/audio in social media, etc.

Real-world applications rely on real-time feedback~(\eg, transferring image to audio for people with visual impairment, generating context caption of camera feed for robot). In such application scenarios, the responsiveness of NICG models is crucial. However, existing NICG techniques mainly focus on improving model accuracy or defending the adversarial accuracy-based attacks~\cite{gan17cvpr,vaswani17nips,vinyals15cvpr,wang20mm,chiaro20nips}. Whether the NICG model can maintain efficiency under adversarial pressure is still a blank domain.

In order to study the efficiency robustness of NICG models, the first thing we need to do is to figure out what factors will affect NICG model efficiency.
In this paper, we investigate a natural property of NICG models. 
The NICG model producing output tokens is a Markov Process; hence the number of underlying decoder calls is non-deterministic. Thus, the computational consumption of NICG models is naturally non-deterministic. 
This natural property discloses a potential vulnerability of NICG models. Adversaries may be able to design specific adversarial inputs to increase computational cost in NICG models significantly. 
Such efficiency vulnerability could lead to severe outcomes in real-world scenarios.
For example, efficiency-based attacks may cause a large magnitude of redundant computational resources and affect the user experience, such as increasing the device battery consumption or extending the response latency.
In this paper, we plan to investigate such potential vulnerability by answering the following questions:

\begin{center}
\begin{tcolorbox}[colback=gray!10,%gray background
                  colframe=black,% black frame colour
                  width=8.4cm,% Use 8cm total width,
                  arc=1mm, auto outer arc,
                  boxrule=0.9pt,
                 ]
\textit{Can we make unnoticeable modifications to image inputs to significantly increase the computational consumption of NICG models and degrade the model efficiency? If so, how severe the efficiency degradation can be?}
\end{tcolorbox}
\end{center}

Existing work on adversarial machine learning (ML)~\cite{goodfellow15iclr,dezfooli16cvpr,kurakin17iclr,carlini17oakland,jang17acsac,madry18iclr,chen18aaai,rony19cvpr} can not answer the aforementioned questions because of the following two reasons: \textit{(i)} existing adversarial attacks mainly focus on the classification DNN model, whose output is a deterministic numeric vector representing the likelihood for different categories. In contrast, our target model is the NICG model, whose output generation process is a non-deterministic Markov process, and the output is a sequence of numeric vectors.
Existing accuracy-based adversarial ML techniques can not handle the dependency in the Markov process.
Furthermore, \textit{(2)} the goal of efficiency robustness evaluation is to increase the computational cost to detect the possible computational resources leakage while existing accuracy-based work seeks to maximize the DNNs errors.
The natural difference between these two goals requires a totally new design of the optimization function for efficiency robustness evaluation.

In this paper, we propose a new methodology, \tool, to generate efficiency-oriented adversarial inputs for evaluating the NICG model efficiency robustness.
These adversarial inputs contain unnoticeable perturbations and consume more computation resources than original inputs in NICG models. To be specific, \tool will apply the minimal perturbation on the benign inputs that could minimize the likelihood of End Of Sentence~(EOS) token and delay the appearance of EOS accordingly.

\noindent\textbf{Evaluation.}
To evaluate the effectiveness of \tool, we perform \tool on four subject models with two datasets, Flickr8k~\cite{Flickr8k}, and MS-COCO~\cite{lin2014microsoft}. 
We compare \tool against six baseline techniques, including two accuracy-based attack algorithms and four natural image corruptions. 
To represent the efficiency degradation severity, we define I-Loops and I-Latency metrics to measure the increment of the decoder calls of the target models and CPU/GPU response latency caused by \tool and baselines. The evaluation results show that \tool has achieved performance far exceeding all baselines on all subjects, increasing the loop numbers, CPU/GPU latency of NICG model up to 483.86\%, 198.76\% and 290.40\% respectively. 

\noindent\textbf{Contribution.}
Our contributions are formalized as below:
\begin{itemize}
    \item We state a new vulnerability of NICG models. The computational consumption of NICG models is volatile for different inputs, thus the adversaries can decrease the efficiency of NICG models by increasing the computational resource consumption.
    \item   We propose a new methodology to evaluate the efficiency robustness of NICG models. To the best of our knowledge,  \tool is the first technique to measure the efficiency robustness for NICG models.
    \item We evaluate \tool on four subject models with two popular datasets and compare with six baselines. The evaluation results show that it's necessary to improve and protect the efficiency robustness of NICG models.
\end{itemize}

\section{Background}

\subsection{Neural Image Caption Generation Model}
\label{sec:background}

\begin{figure}[hb]
    \centering
    \includegraphics[width=0.5\textwidth]{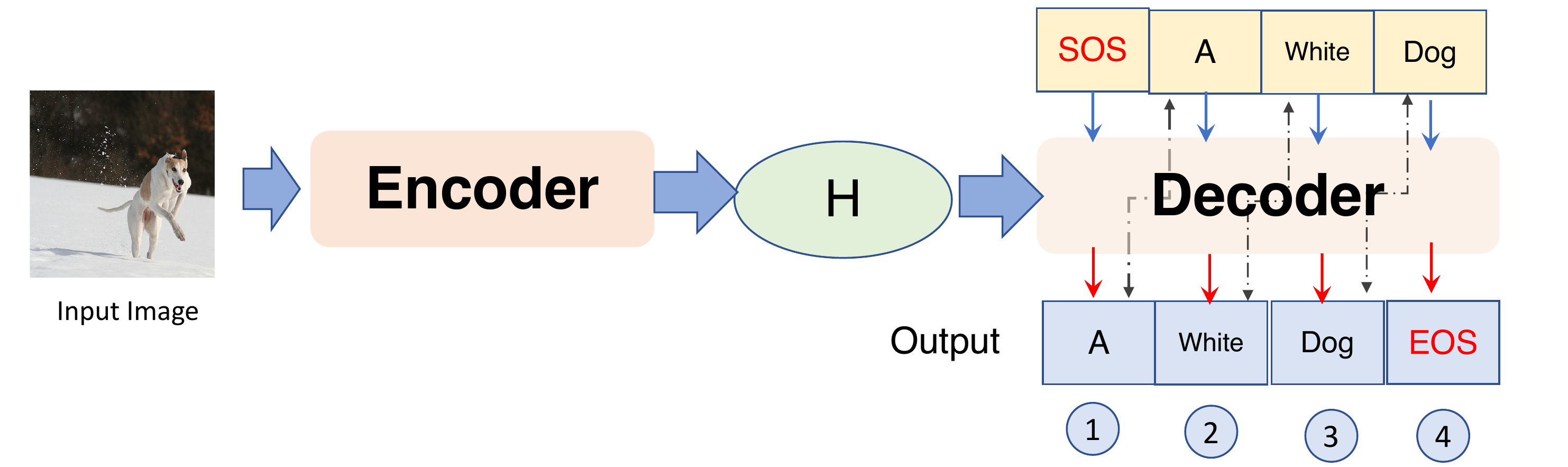}
    \caption{Working mechanism of neural image caption generator}
    \label{fig:background}
\end{figure}

Neural Image Caption Generation~(NICG)~\cite{anderson18cvpr,gan17cvpr,vaswani17nips,vinyals15cvpr,wang20mm,chiaro20nips,cornia20cvpr,huang19iccv,pan20cvpr} model calculates the conditional probability $P(Y | X)$, where $X$ is the input image and $Y = [y_1, y_2, \cdots, y_n]$ is the target token sequences that will be used as image captions. 
As shown in \figref{fig:background}, the input image is first sent through the encoder $\mathcal{F}_{encoder}$ to produce the hidden representation $H$.
After that, starting with a special token~(SOS), the decoder $\mathcal{F}_{decoder}$ uses $H$ in an iterative way for an auto-regressive generation of output tokens $Y$. 
The tokens are generated one by one until the process reaches the end of sequence (EOS) token or a pre-set maximum length. 
As the process is iterative, NICG models' computational resources consumption is proportional to the length of generated output sequence. Therefore, a longer output sequence would make the model less efficient.

\subsection{DNNs Efficiency}
The accuracy and complexity of DNN models are positively correlated. Excellent model accuracy often implies a large number of neural layers and complex model construction, followed by huge inference-time computational cost and low efficiency. To reduce DNNs inference-time cost and faster the inference processes for real-time applications, many related works have been proposed.
The related work can be divided into two types, The first type~\cite{howard17corr,zhang18cvpr} prunes DNN models offline by identifying and removing the unimportant/redundant neurons. The second type~\cite{wang18eccv,graves16corr,figurnov17cvpr} reduces the number of computations online by dynamically skipping the unnecessary part of DNNs, known as input-adaptive techniques.
Even though the input-adaptive techniques balance the model accuracy with computational costs, this balance is not robust. According to the recent studies~\cite{haque20cvpr,hong21iclr, haque2022ereba, haque2021nodeattack, chen2021transslowdown}, the input-adaptive DNN models are not robust against the adversarial attack, \ie, these techniques cannot lower computational costs under adversarial scenarios.

\subsection{Adversarial Attacks}
The adversarial example refers to an intentionally modified version of the benign example (\eg, adding perturbations). With the human-unnoticeable perturbations, the adversarial example could fool even the state-of-the-art DNNs~\cite{carlini17oakland, athalye18icml,szegedy14iclr}.
Normally, adversarial examples can be generated by performing perturbation that follows adversarial gradients~\cite{madry18iclr} or optimizes perturbation with given loss \cite{chen18aaai}.
The perturbation will be constrained by magnitude, among which $L_{2}$-Norm and $L_{inf}$-Norm are the most commonly used ones~\cite{carlini17oakland,madry18iclr}.
According to the difference of prior knowledge on the victim DNN model, the adversarial example generation techniques could be categorized into the white-box attack and black-box attack~\cite{goodfellow15iclr,dezfooli16cvpr,kurakin17iclr,carlini17oakland,jang17acsac,madry18iclr,chen18aaai,rony19cvpr}.

\section{Preliminary}

\begin{figure}[htbp!]
    \centering
    \includegraphics[width=0.48\textwidth]{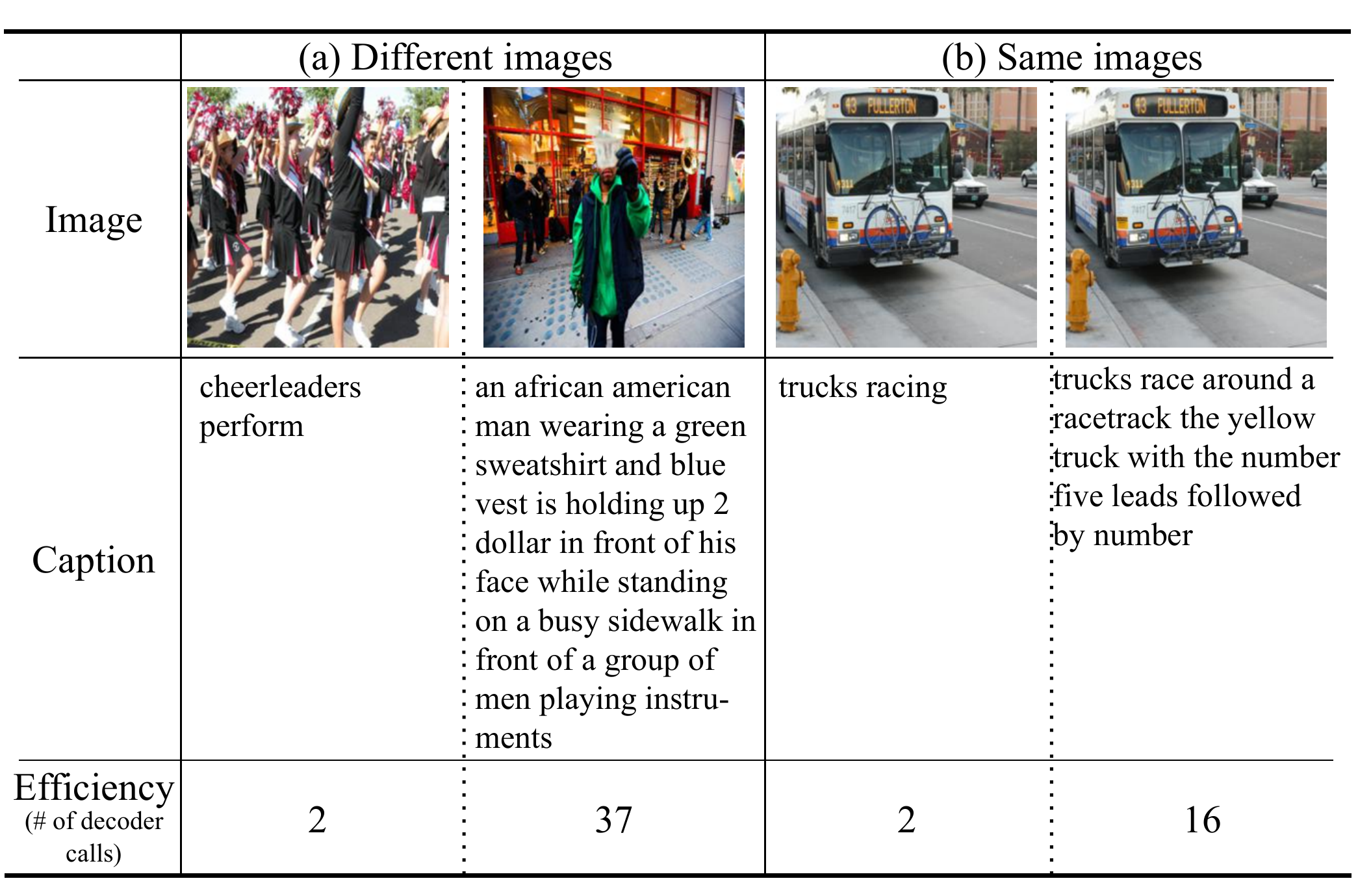}
    \caption{Efficiency uncertainty for images from MS-COCO}
    \label{fig:compare}
\end{figure}

\label{sec:preliminary}
As we discussed in \secref{sec:background}, NICG models will not terminate until the output token reaches EOS or a pre-set maximum length. 
In this section, we conduct a preliminary study to show that the value of the pre-set maximum length is hard to estimate because of the uncertainty in image caption tasks. 
Specifically, we select three images and corresponding captions from the MS-COCO dataset (shown in \figref{fig:compare}) to show the uncertainty.

\noindent\textbf{Variance across Different Images.}
For different images, the task complexity of analyzing their contents could be completely different. In the image caption task, different image semantics will significantly affect NICG model efficiency. For example, in \figref{fig:compare}~(a), due to the difference in the scene semantics, the corresponding caption lengths of these two images have a huge difference.

\noindent\textbf{Uncertainty in Labelling the Same Image.}
Another challenge to estimate max-length is the uncertainty from the training images. For example, in \figref{fig:compare}~(b), two different versions of the caption for the same training image have different lengths, which will increase the efficiency uncertainty in the NICG models trained with this image. 

Because of the significant variance and uncertainty mentioned above, estimating an exact maximum length for each image is challenging. Thus, a common practice is to set a pretty large value for all images to avoid incomplete captioning~(at least larger than the maximum caption length in the training dataset).

\begin{figure}[h]
    \centering
    \includegraphics[width=0.5\textwidth]{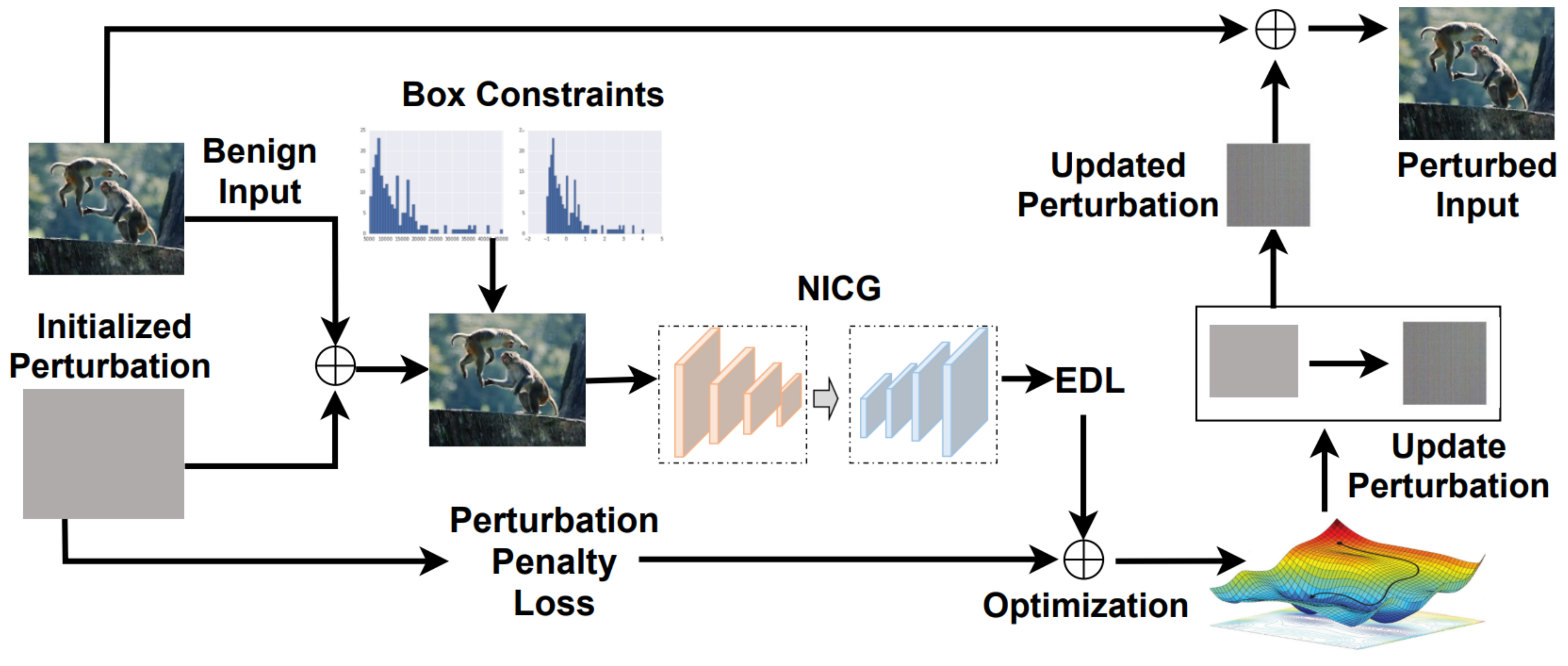}
    \caption{NICG workflow}
    \label{fig:nicg}
\end{figure}

\section{Approach}
\subsection{Problem Formulation}
\begin{equation}
\label{eq:define}
    \begin{split}
        & \quad \Delta = \text{argmax}_{\delta} \quad \text{Loop}_{\mathcal{F}}(x + \delta) \\ 
        & s.t. \quad ||\delta|| \le \epsilon \; \wedge \; ||x + \delta|| \in [0, 1]^n \\
    \end{split}
\end{equation}
Our objective is to generate human-unnoticeable perturbations to images to decrease the victim NICG model efficiency during inference. 
Specifically, our objective concentrates on three factors: \textit{(i)} the generated adversarial image should increase the victim NICG model computational complexity; 
\textit{(ii)} the generated adversarial image $x'$ can not be differentiated by humans from the benign image $x$; \textit{(iii)} the generated adversarial image $x'$ should be realistic in the real world.
We formulate the mentioned three factors in \equref{eq:define}. 
In \equref{eq:define}, $x$ is the benign input, $\mathcal{F}$ is the victim NICG model under attack, $\epsilon$ is the maximum adversarial perturbation allowed, and $\text{Loop}_{\mathcal{F}}(\cdot)$ measures the number of decoder calls in the victim NICG model $\mathcal{F}$.
Our proposed approach \tool tries to search for an optimal perturbation $\Delta$ that maximizes the number of decoder calls while holding the constraints that perturbation is smaller than the allowed threshold~(unnoticeable) and existing in the real world~(realistic).

\subsection{Attack Overview}

\figref{fig:nicg} shows the overview of our proposed attack.
Given a benign input image, \tool first initializes an adversarial perturbation satisfying the realistic box constraints~(\S \ref{sec:box_constraints}). After that, \tool computes the efficiency reduction loss~(\S \ref{sec:degradation_loss}) and the perturbation penalty loss~(\S \ref{sec:perturbation_loss}).
The reduction loss aims to slow down the victim NICG model, and the perturbation penalty loss seeks to enforce the generated adversarial examples to satisfy the unnoticeable constraints in  \equref{eq:define}.
Finally, \tool updates the adversarial perturbation by jointly optimizing the perturbation penalty loss and the efficiency reduction loss.

\subsection{Detail Design}

\subsubsection{Realistic Box Constraints}
\label{sec:box_constraints}
\begin{equation}
\label{eq:variable}
    \delta = \frac{1}{2}(tanh(w) + 1) - x
\end{equation}
To ensure the adversarial example is a valid image, we constraint the adversarial perturbation $\delta$ in \equref{eq:define}: $||x + \delta|| \in [0, 1]^n$.
Such constraints are known as box constraints in the optimization theory~\cite{carlini17oakland}.
To satisfy the constraints, instead of directly optimizing $\delta$, we introduces a new variable $w$ and apply a change-of-variables to optimize over $w$. The relationship between $w$ and $\delta$ is shown in \equref{eq:variable}. Because the range of function $tanh(\cdot)$ is $[-1, 1]$, $\delta + x$ will always satisfy the constraint  $||x + \delta|| \in [0, 1]^n$.

\subsubsection{Efficiency Reduction Loss}
\label{sec:degradation_loss}

\begin{figure}[h]
    \centering
    \includegraphics[width=0.99\linewidth]{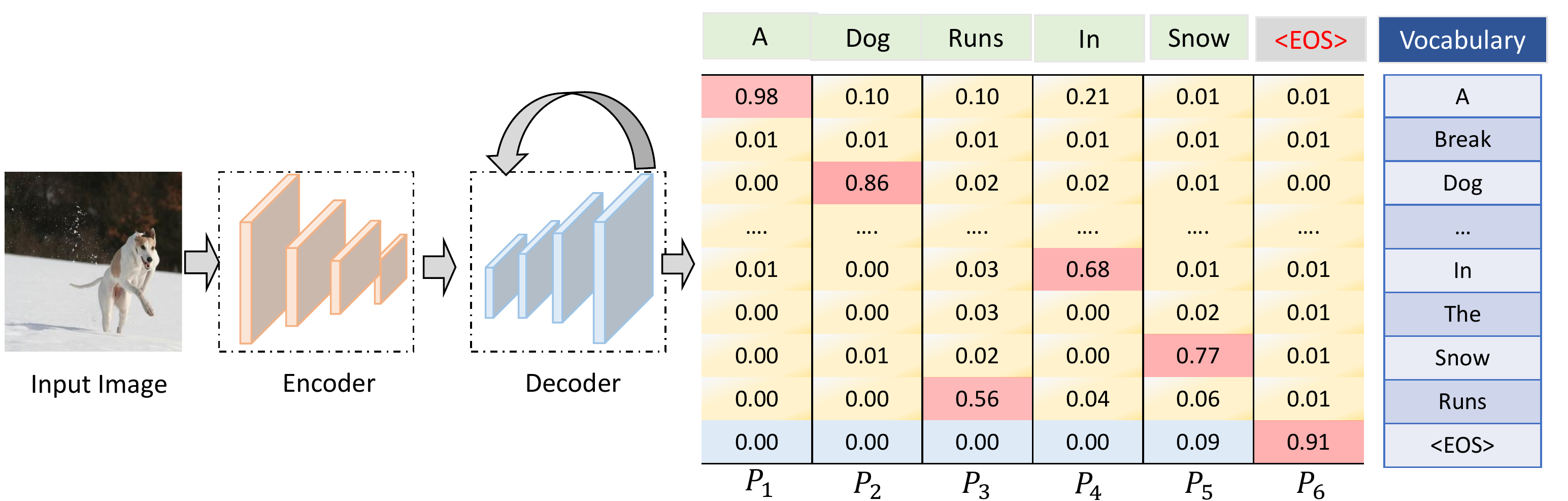}
    \caption{Distribution of output tokens}
    \label{fig:adv_loss}
\end{figure}

% As shown in \figref{fig:adv_loss}, NICG model’s output is a sequential of probability distributions with undetermined length instead of one probability distributions with constant length.
As we discussed in \S \secref{sec:background}, NICG model efficiency is related to the likelihood of the EOS tokens.
Thus, to degrade the efficiency of NICG model, our intuition is to decrease the EOS tokens likelihood.
Formally, let NICG model's output be a sequence of probability distributions, \ie  $\mathcal{F}(x) = [p_1, p_2,\cdots, p_n]$ and the output token sequences are $[o_1, o_2, \cdots, o_n]$, where $o_i =\text{argmax} (p_i)$.
Then we denote the likelihood of the output tokens and the EOS tokens as $[p_1^{o_1}, p_2^{o_2}, \cdots, p_n^{o_n}]$ and $[p_1^{eos}, p_2^{eos}, \cdots, p_n^{eos}]$ respectively. 
In the example of \figref{fig:adv_loss}, we have
\begin{equation*}
    \begin{split}
        [p_1^{o_1}, p_2^{o_2}, \cdots, p_n^{o_n}] = [0.98, 0.86, 0.56, 0.68, 0.77, 0.91] \\
        [p_1^{eos}, p_2^{eos}, \cdots, p_n^{eos}] =  [0.00, 0.00, 0.00, 0.00, 0.09, 0.91] \\
    \end{split}
\end{equation*}
Then our efficiency reduction objective can be divided into two parts: \textit{(i)} Delay EOS appearance and \textit{(ii)} Break output dependency.

\noindent\textbf{Minimize EOS Probability.} To delay the appearance of EOS tokens, existing work usually applies minimum likelihood estimation~(MLE) to minimize the likelihood of EOS tokens.
However, as the NICG model vocabulary is normally pretty large~(more than 1,000), MLE becomes inefficient because MLE requires to compute the cross-entropy loss, which is inefficient on large vocabulary.
To address the limitation of inefficiency cross-entropy on large vocabulary, we borrow the idea of noise contrastive estimation~(NCE)~\cite{gutmann2010noise} and design our loss function.
Specifically, we treat the probability distribution $p_i$ for the multi-classification task as a binary classification task \ie,  is or not an EOS token.
We then define a new probability distribution $q_i = [l_i^{eos}, \sum_j l_i^j - l_i^{eos}]$ to represent the logits distribution of the proposed binary classification task.
Finally, our goal to delay the appearance of the EOS token can be formulated as \equref{eq:eos_loss}.
\begin{equation}
\label{eq:eos_loss}
    \mathcal{L}_{eos} = \frac{1}{n}\sum_{i=1}^{n} \left\{  l_i^{eos} - \mathbb{E}_{k\sim p_i} l_i^{k}
\right\}
\end{equation}
With the help of logits $l_i^j$, we do not need to compute the softmax on large vocabulary thus could compute the objective function more efficiency.
Next, we prove that our objective function \equref{eq:eos_loss} will convergence to  MLE method's loss function \ie, $L = \frac{1}{n}\sum_{i=1}^{n} \text{log}p_i^{eos}$ .

\begin{lemma} The proposed loss function $\mathcal{L}_{eos}$ will finally convergence to the MLE method's objective function $L= \frac{1}{n}\sum_{i=1}^{n} \text{log}p_i^{eos}$. 

\end{lemma}

\begin{proof} Denote the logits on the $i^{th}$ token as $[l_{i}^{1}, l_{i}^{2}, \cdots, l_{i}^{V}]$, where $V$ is the size of the vocabulary, then we have  
$$p_i^j = \frac{\text{exp}(l_{i}^{j})}{\sum_{k=1}^{V} \text{exp}(l_{i}^{k})}$$
MSE seeks to minimize the likelihood of EOS, then the objective function is
$$L = \frac{1}{n}\sum_{i=1}^{n} \text{log}p_i^{eos} = \frac{1}{n}\sum_{i=1}^{n} \{l_i^{eos} - \text{log} \sum_{k=1}^{V} \text{exp}(l_{i}^{k})\}$$
the gradients of the above objective is 
$$\frac{\partial L}{\partial x} = \frac{1}{n}\sum_{i=1}^{n} \left\{ \frac{\partial l_i^{eos}}{\partial x} -   \frac{\text{exp}(l_i^{k})}{\sum_{k=1}^{V} \text{exp}(l_{i}^{k})} \sum_{k=1}^{V}\frac{\partial l_i^{k}}{\partial x}
\right\}$$
Notice that $p_i^j = \frac{\text{exp}(l_{i}^{j})}{\sum_{k=1}^{V} \text{exp}(l_{i}^{k})}$, then we have 

\begin{equation}
    \begin{split}
        & \frac{\partial L}{\partial x} = \frac{1}{n}\sum_{i=1}^{n} \left\{ \frac{\partial l_i^{eos}}{\partial x} - \sum_{k=1}^{V} p_i^{k} \frac{\partial l_i^{k}}{\partial x} \right\}  \\
& \quad\quad = \frac{1}{n}\sum_{i=1}^{n} \left\{ \frac{\partial l_i^{eos}}{\partial x} - \mathbb{E}_{k\sim p_i}\frac{\partial l_i^{k}}{\partial x}
\right\} \\
&\quad\quad = \frac{\partial\mathcal{L}_{eos}}{\partial x}
    \end{split}
\end{equation}
\end{proof}

Because of the convergence of Monte Carlo method, we prove Lemma 1.

\noindent\textbf{Break output Dependency.} Because the token generation process of NICG models is a 
Markov process, \ie , NICG model outputs the probability distribution $p_i$ based on the previous output token $o_{i - 1}$, \ie, $p_i =\mathcal{F}_{decoder}(o_{i-1}, h)$. 
Minimize $p_i^{eos}$ may not change the output tokens at the positions from 0 to $n-1$. Thus minimizing $p_n^{eos}$ will be challenging because the previous token $o_{n-1}$ keeps the same.
To accelerate the process of delaying EOS tokens, we seek to break such output dependency.
Similar to the objective in delaying EOS appearance, we have the objective in \equref{eq:o_loss}.

\begin{equation}
\label{eq:o_loss}
    \mathcal{L}_{dep} = \frac{1}{n}\sum_{i=1}^{n} \left\{  l_i^{o_i} - \mathbb{E}_{k\sim p_i} l_i^{k}
\right\}
\end{equation}

\noindent\textbf{Final Efficiency Reduction Objective.}
Our final efficiency reduction objective can be formulated as \equref{eq:adv_loss}, which aims to delay the EOS token appearance and break the output dependency.

\begin{equation}
\label{eq:adv_loss}
    \mathcal{L}_{deg} = \mathcal{L}_{eos} + \lambda \mathcal{L}_{dep}
\end{equation}

\subsubsection{Perturbation Penalty Loss}
\label{sec:perturbation_loss}

\begin{equation}
\label{eq:per_loss}
    \mathcal{L}_{per} = \left\{
    \begin{aligned}
        &\quad 0 ;   \quad\quad\quad  \text{if} \; \delta \le \epsilon \\
        & ||\delta - \epsilon||; \quad \text{otherwise}
    \end{aligned}
     \right.
\end{equation}
To ensure that the adversarial example will be unnoticeable to humans, we constraint the magnitude of the adversarial perturbation in \equref{eq:define}, \ie,  $||\delta|| \le \epsilon$. 
To achieve such goal, we introduce the perturbation penalty loss in \equref{eq:per_loss}, if the adversarial perturbation $\delta$ is less than the allowed perturbation magnitude, the penalty is zero, otherwise, the penalty will increase linearly as $||\delta - \epsilon||$ increases.

\subsection{Attack Algorithm}

\begin{algorithm}[btp]
\caption{\tool Attack} 
\label{alg:attack}

 {\bf Input:} Benign input $x$ \\
 {\bf Input:} Victim NICG model $\mathcal{F}(\cdot)$ \\
 {\bf Input:} Maximum perturbation $\epsilon$  \\
{\bf Input:} Maximum Iterations T  \\
 {\bf Output:} Adversarial examples $x'$ that satisfy \equref{eq:define}\\
\begin{algorithmic}[1]
% \STATE $x' \Leftarrow x$ \xspace\xspace\qquad\qquad\qquad Initialize $x'$ with $x$.
\STATE $\delta \Leftarrow 0$ \xspace\xspace\xspace\qquad\qquad\qquad Initialize $\delta$ with 0.
\STATE $w \Leftarrow arctanh(2x - 1)$ Initialize $w$ based on \equref{eq:variable}.

\FOR{$iter \; \text{in} \; \text{Range}$(T)}
        \STATE$x'= \frac{1}{2}(tanh(w) + 1)$  Compute $x'$ based on \equref{eq:variable}
        \STATE $\delta = x' - x$ \quad Compute the perturbation magnitude
        \STATE $\mathcal{L}_{deg}=L_1(x', \mathcal{F})$ Compute $\mathcal{L}_{per}$ according to \equref{eq:adv_loss}.
        \STATE $\mathcal{L}_{per} = L_2(\delta, \epsilon)$ Compute $\mathcal{L}_{per}$ according to \equref{eq:per_loss}.
        \STATE $\mathcal{L}_{total} = \mathcal{L}_{deg} + \lambda\mathcal{L}_{per}$ Compute joint loss.
        \STATE $\triangledown = \frac{\partial \mathcal{L}_{total}}{\partial w}$ Compute the gradients
        \STATE $w = w + lr \times \triangledown$ Update the latent variable $w$.
\ENDFOR
\STATE Return $\frac{1}{2}(tanh(w) + 1)$ Return the adversarial example.

% \FOR{each $x \in \mathcal{X}_{normal}$}
% \STATE  counter-attack $x$ with $\Delta_{know}$
% \IF{$\Delta_{know}(x) > r_2$}
% \STATE $i = i +1$      \qquad the times that larger than $r_2$
% \ELSE 
% \STATE $j = j +1$       \qquad the times that smaller than $r_2$
% \ENDIF
% \ENDFOR
% \STATE  Return $i / (i + j)$  \quad return the frequency $\Delta_{know}(x) > r_2$
\end{algorithmic}
\end{algorithm}

The attack algorithm is shown in Algorithm \ref{alg:attack}. Our attack algorithm accepts four inputs: a benign input image $x$, the victim NICG model $\mathcal{F}$, a pre-defined perturbation threshold $\epsilon$, and the maximum iteration number T.
Our algorithm outputs an adversarial example $x'$ that satisfy \equref{eq:define}.
Our algorithm first initializes the adversarial perturbation $\delta$ as zero and compute the corresponding $w$~(line 1 and 2).
After that, we iteratively update the latent variable $w$. Specifically, we compute the efficiency reduction loss $\mathcal{L}_{deg}$ based on \equref{eq:adv_loss} and the perturbation penalty loss based on \equref{eq:per_loss}.
We then optimize $w$ by minimizing the joint losses.
After iteration, we transform the latent variable $w$ back to image space and return the adversarial example.

\section{Evaluation}

\subsection{Experimental Setup}

% \begin{table}[htbp]
%   \centering
%   \caption{Experimental Subjects}
%   \resizebox{.49\textwidth}{!}{
%     \begin{tabular}{c|c|cc|c|c|c}
%     \toprule
%     \multirow{2}[2]{*}{Dataset} & \multirow{2}[2]{*}{No.}& \multicolumn{2}{c|}{Model} & \multirow{2}[2]{*}{Train} & \multirow{2}[2]{*}{Valid} & \multirow{2}[2]{*}{Test} \\
%         &  & Encoder & Decoder &       &       &  \\
%     \hline
%     \multirow{2}[1]{*}{Flickr8k} & A & ResNext & Attention + LSTM & 6000  & 1000  & 1000 \\
%           & B & GoogLeNet & Attention + RNN & 6000  & 1000  & 1000 \\ \hline
%     \multirow{2}[1]{*}{MSCOCO} & C & MobileNets & Attention + LSTM & 82783 & 40504 & 40775 \\
%           & D & ResNet & Attention + RNN & 82783 & 40504 & 40775 \\
%     \bottomrule
%     \end{tabular}%
%     }
%   \label{tab:data}%
% \end{table}%

\begin{table}[htbp]
  \centering
  \caption{Experimental Subjects}
  \resizebox{.49\textwidth}{!}{
\begin{tabular}{c|c|cc|c|c|c}
\toprule[1.2pt]
\multirow{2}{*}{Dataset}  & \multirow{2}{*}{Subject} & \multicolumn{2}{c|}{Model}                                               & \multirow{2}{*}{Train} & \multirow{2}{*}{Valid} & \multirow{2}{*}{Test} \\
                          &                          & Encoder    & Decoder                                                     &                        &                        &                       \\ \hline
\multirow{2}{*}{Flickr8k} & A                        & ResNext    & \begin{tabular}[c]{@{}c@{}}Attention\\  + LSTM\end{tabular} & 6000                   & 1000                   & 1000                  \\ \cline{2-7} 
                          & B                        & GoogLeNet  & \begin{tabular}[c]{@{}c@{}}Attention\\  + RNN\end{tabular}  & 6000                   & 1000                   & 1000                  \\ \hline
\multirow{2}{*}{MS-COCO}   & C                        & MobileNets & \begin{tabular}[c]{@{}c@{}}Attention\\  + LSTM\end{tabular} & 82783                  & 40504                  & 40775                 \\ \cline{2-7} 
                          & D                        & ResNet     & \begin{tabular}[c]{@{}c@{}}Attention\\  + RNN\end{tabular}  & 82783                  & 40504                  & 40775                 \\ 
                          \bottomrule[1.2pt]
\end{tabular}
}
\label{tab:data}
\end{table}
\noindent\textbf{Models and Datasets.}
We evaluate our proposed technique~\footnote{Our code is available at \href{https://github.com/SeekingDream/NICGSlowDown}{https://github.com/NICGSlowDown}} on two public datasets, Flickr8k~\cite{Flickr8k}, and MS-COCO~\cite{lin2014microsoft}.
\tabref{tab:data} shows the detail of NICG models for each corresponding dataset.
Flickr8k dataset contains 8,000 images~(including 6,000 training images, 1,000 validation images and 1,000 test images).
We apply two encoder-decoder models for the Flickr8k dataset. The first one applies ResNext~\cite{resnext17cvpr} as encoder and LSTM module as decoder~\cite{lstm97}. The second one applies GoogLeNet~\cite{googlenet15cvpr} as encoder and RNN as decoder~\cite{rnn86nature}.
MS-COCO dataset contains
123,287 images~(including 82,783 training images, 40,504 validation images and 40,775 testing images).
We also apply two encoder-decoder models for the MS-COCO dataset. The first one is MobileNets~\cite{howard17corr} + LSTM and the latter one is ResNet~\cite{resnet16cvpr} + RNN.

\begin{equation}
\label{eq:metric}
    \begin{split}
        &\quad \quad  \text{I-Loop} = \frac{\text{Loop}(x') - \text{Loop}(x)}{\text{Loop}(x)} \times 100\% \\
        & \text{I-Latency} = \frac{\text{Latency}(x') - \text{Latency}(x)}{\text{Latency}(x)} \times 100\% \\
    \end{split}
\end{equation}

\noindent\textbf{Metrics.} We select two metrics, the number of decoder calls and response latency, to represent the efficiency of NICG models.
As we discussed in \S \ref{sec:background}, higher decoder calls indicate that the NICG model cast more floating-point operations~(FLOPs) to handle the input image, which leads to less efficiency~\cite{zhang18cvpr,wang18eccv}.
Response latency is a hardware-dependent metric used to measure NICG model runtime efficiency. High response latency indicates worse real-time caption quality and higher battery consumption. We measure the response latency on two hardware platforms: Intel Xeon E5-2660v3 CPU and Nvidia1080Ti GPU.
Specifically, we define two metrics, I-Loop and I-Latency, to show the effectiveness of \tool in degrading the NICG model efficiency.
The formal definition of I-Loop and I-Latency are shown in \eqref{eq:metric},  where $x$ and $x'$ denotes the benign example and the generated adversarial example respectively, Loop($\cdot$) and Latency($\cdot$) are the functions to calculate the decoder calls and response latency respectively. Higher I-Loop and I-Latency refer to more severe efficiency slowdown caused by the adversarial example.

\noindent\textbf{Comparison Baselines.}
To the best of our knowledge, we are the first to study the efficiency robustness of NICG models; therefore, no existing efficiency attacks can be applied as our baselines.
To show that existing accuracy-based methods can not be applied to evaluate the NICG model's efficiency robustness, we compare \tool against two accuracy-based attack algorithms and four natural image corruptions.
Specifically, we choose PGD~\cite{madry2018towards} and  CW~\cite{carlini17oakland}) as the accuracy-based attack algorithms and image quantization~\cite{xu2017feature, hendrycks2019benchmarking}, Gaussian noise~\cite{xu2017feature,hendrycks2019benchmarking}, JPEG compression~\cite{liu2019feature} and feature squeezing~\cite{xu2017feature} as the corruptions.

\noindent\textbf{Implementation Details. }
We follow \cite{xu2015show} to implement the four neural image caption generation models.
We set the NICG model's maximum caption length as 60 as the maximum caption length in the training dataset is 53.
We filter out the tokens with frequencies less than 5. 
Finally, our vocabulary sizes are 2,633 and 11,569 for Flick8k and MS-COCO datasets.
We implement \tool with Pytorch and set the maximum perturbation $\epsilon$ as 40 and 0.03 for $L_{2}$ and $L_{inf}$ adversarial examples.
We set maximum iteration $T$ as 1,000 and the hyper-parameter $\lambda$ as $1.0\times10^{4}$.

\subsection{Effectiveness and Severity}

\noindent\textbf{Effectiveness of Attack.} \figref{fig:distribution} shows the distribution of efficiency metrics for Subject A~(more results are shown in Appendix). 
The first and second rows represent the Probability Density Function~(PDF) and Cumulative Distribution Function~(CDF) results. For convenience, we reverse the CDF from one to zero.
The area under the CDF curve indicates the efficiency of the NICG model, and a larger area indicates the NICG model is less efficient.
The green area denotes the distribution of benign examples,  and the red represents the distribution of adversarial examples generated by \tool.
From \figref{fig:distribution}, we could observe that adversarial examples significantly change the number of decoder calls and latency distribution in the NICG model.
This observation indicates that our attack could effectively slow down the NICG model.

\begin{figure}[h]
    \centering
    \includegraphics[width=0.5\textwidth]{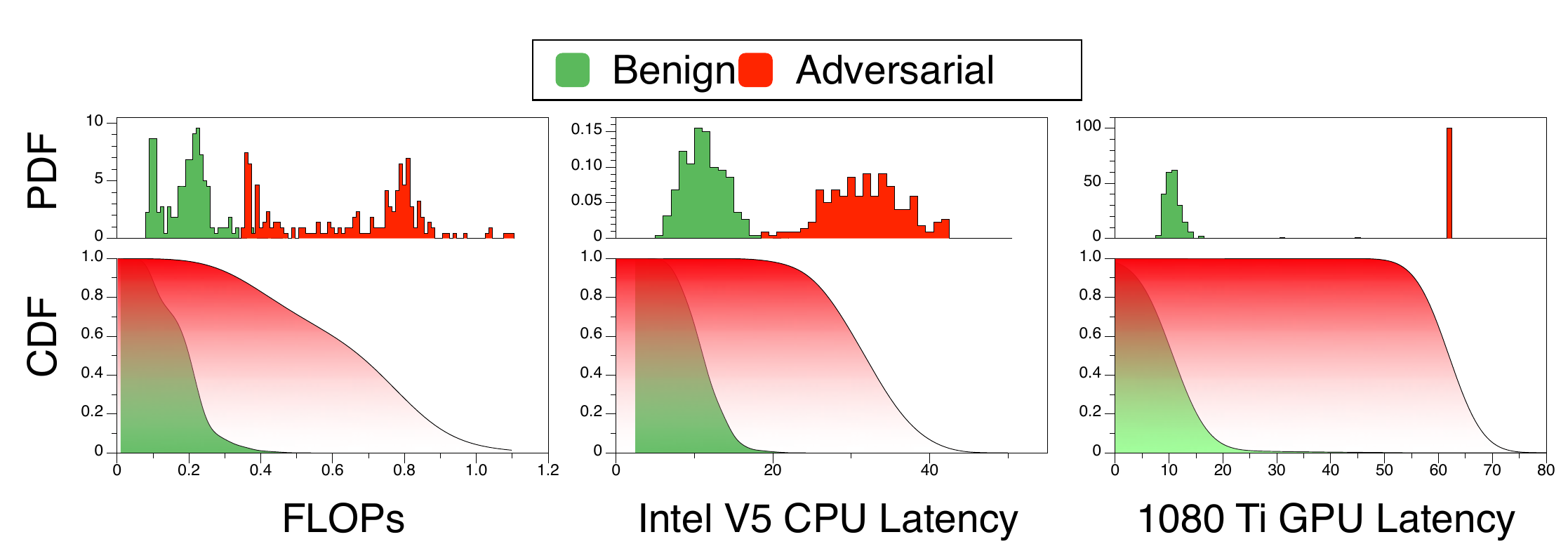}
    \caption{Efficiency distribution of benign and adversarial examples (More results can be found in appendix)}
    \label{fig:distribution}
\end{figure}

% Table generated by Excel2LaTeX from sheet 'main'
\begin{table*}[htbp]
  \centering
  \caption{Resutls of efficiency reduction}
  \resizebox{0.78\textwidth}{!}{
    \begin{tabular}{c|c|c|rrrrrr|r}
    \toprule
    \multicolumn{1}{l|}{Subject} & Norm  & Metric & \multicolumn{1}{c}{PGD} & \multicolumn{1}{c}{CW} & \multicolumn{1}{c}{Quantize} & \multicolumn{1}{c}{Gaussian} & \multicolumn{1}{c}{JPEG} & \multicolumn{1}{c|}{TVM} & \multicolumn{1}{c}{Ours} \\
    \hline
    \multirow{6}[4]{*}{A} &       & I-Loop & 87.30  & 7.26  & 3.97  & 0.47  & -5.25  & -0.12  & \textbf{483.86 } \\
          & L2    & I-Latency(CPU) & 36.33  & 2.59  & 3.00  & 1.44  & -3.28  & -1.51  & \textbf{198.76 } \\
          &       & I-Latency(GPU) & 75.47  & 18.93  & 11.81  & 13.35  & 8.47  & 15.60  & \textbf{290.40 } \\
\cline{2-10}          &       & I-Loop & 7.42  & 44.88  & 6.37  & 0.79  & -5.25  & -0.12  & \textbf{354.11 } \\
          & Linf  & I-Latency(CPU) & 3.86  & 16.68  & 2.43  & 1.82  & -15.45  & 9.06  & \textbf{202.81 } \\
          &       & I-Latency(GPU) & 34.62  & 38.18  & 18.96  & 10.14  & 14.75  & 21.64  & \textbf{241.90 } \\
    \hline
    \multirow{6}[4]{*}{B} &       & I-Loop & 18.62  & 6.66  & 11.94  & 5.40  & 2.14  & 0.45  & \textbf{481.32 } \\
          & L2    & I-Latency(CPU) & 2.77  & 3.73  & 24.65  & -14.78  & 1.65  & 5.80  & \textbf{87.37 } \\
          &       & I-Latency(GPU) & 26.16  & 15.31  & 14.55  & 1.65  & 2.47  & 6.12  & \textbf{223.38 } \\
\cline{2-10}          &       & I-Loop & 8.22  & 36.74  & 10.33  & 0.04  & 2.14  & 0.45  & \textbf{271.19 } \\
          & Linf  & I-Latency(CPU) & 1.86  & 6.32  & -0.74  & 3.28  & 2.55  & 2.53  & \textbf{8.32 } \\
          &       & I-Latency(GPU) & 6.94  & 20.99  & 22.55  & 12.56  & 3.20  & 14.11  & \textbf{71.77 } \\
    \hline
    \multirow{6}[4]{*}{C} &       & I-Loop & 48.48  & 2.76  & -0.08  & 7.96  & -5.69  & -0.96  & \textbf{433.58 } \\
          & L2    & I-Latency(CPU) & 33.07  & 3.94  & -0.17  & 3.32  & -2.19  & -11.01  & \textbf{155.61 } \\
          &       & I-Latency(GPU) & 32.14  & 14.51  & 13.75  & 19.19  & 2.34  & 9.98  & \textbf{297.37 } \\
\cline{2-10}          &       & I-Loop & -5.97  & 62.89  & 3.49  & -2.32  & -5.69  & -0.96  & \textbf{379.81 } \\
          & Linf  & I-Latency(CPU) & -9.33  & 20.48  & 1.54  & -1.21  & -3.60  & 8.19  & \textbf{90.73 } \\
          &       & I-Latency(GPU) & 6.23  & 29.24  & 16.97  & 13.08  & 20.19  & 21.06  & \textbf{211.41 } \\
    \hline
    \multirow{6}[4]{*}{D} &       & I-Loop & 19.07  & 11.17  & -7.33  & 0.24  & -3.53  & 0.17  & \textbf{408.90 } \\
          & L2    & I-Latency(CPU) & 8.14  & 7.07  & -4.62  & -1.97  & 3.26  & -2.45  & \textbf{155.49 } \\
          &       & I-Latency(GPU) & 31.95  & 17.32  & 8.08  & 31.16  & 21.59  & 15.87  & \textbf{192.58 } \\
\cline{2-10}          &       & I-Loop & 7.82  & 74.09  & -8.35  & 1.51  & -3.53  & 0.17  & \textbf{115.02 } \\
          & Linf  & I-Latency(CPU) & -1.13  & 29.04  & -4.17  & 2.53  & -3.53  & -3.16  & \textbf{21.45 } \\
          &       & I-Latency(GPU) & 23.29  & 43.41  & 10.27  & 3.25  & 3.85  & 7.63  & \textbf{55.36 } \\
    \bottomrule
    \end{tabular}%
    }
  \label{tab:mainres}%
\end{table*}%

\noindent\textbf{Impact of Attack.} To evaluate the severity of our proposed attack on reducing the model efficiency, we measure the I-LOOP and I-Latency for the four subjects we mentioned above. \tabref{tab:mainres} shows the results of the adversarial attack on the targeted model. From the table, we could have the following observations: \textit{(i)} Compared to other baselines, \tool achieves the best performance on slowing down the targeted NICG model in all subjects. For example, adversarial examples generated by \tool increase the number of decoder calls, CPU latency, and GPU latency on Subject A up to 483.86\%, 198.76\%, 290.40\% respectively; \textit{(ii)} Unlike \tool, all baseline methods can not ensure degrading efficiency of the NICG model. In some cases, baselines would even speed up the NICG model processing instead. This observation proves that the existing baseline techniques are not suitable for evaluating the efficiency robustness of NICG models; \textit{(iii)} For all subjects, \tool with L2-Norm achieves better performance compared with Linf-Norm. We infer that because the perturbation size of L2-Norm is more suitable for \tool to apply efficiency attack; \textit{(iv)}  For all subjects, the GPU latency increased by \tool is more effective than the CPU delay, implying that \tool is more effective for efficiency attacks on GPU than CPU.

\subsection{Quality of Generated Images}

%\begin{table}[htbp]
%   \centering
%   \caption{Unnoticeable Constraints Satisfaction Rate}
%     \resizebox{0.3\textwidth}{!}{
%     \begin{tabular}{l|cccc}
%     \toprule
%     \multicolumn{1}{l|}{\textbf{Approach}} & \multicolumn{1}{l}{A} & \multicolumn{1}{l}{B} & \multicolumn{1}{l}{C} & \multicolumn{1}{l}{D} \\
%     \midrule
%     \textbf{PGD} & \textbf{1.00 } & \textbf{1.00 } & \textbf{1.00 } & \textbf{1.00 } \\
%     \textbf{CW} & \textbf{1.00 } & \textbf{1.00 } & \textbf{1.00 } & \textbf{1.00 } \\
%     \textbf{Quantization} & 0.00  & 0.00  & 0.00  & 0.00  \\
%     \textbf{Gaussian} & 0.89  & 0.88  & 0.89  & 0.88  \\
%     \textbf{JPEG} & 0.05  & 0.05  & 0.05  & 0.05  \\
%     \textbf{TVM} & \textbf{1.00 } & \textbf{1.00 } & \textbf{1.00 } & \textbf{1.00 } \\
%     \textbf{Ours} & \textbf{1.00 } & \textbf{1.00 } & \textbf{1.00 } & \textbf{1.00 } \\
%     \bottomrule
%     \end{tabular}%
%     }
%   \label{tab:per}%
% \end{table}%

% Table generated by Excel2LaTeX from sheet 'per'

\begin{table}[htbp]
  \centering
  \caption{The size of the adversarial perturbations}
  \resizebox{0.48\textwidth}{!}{
        \begin{tabular}{c|c|rrrr|r}
    \toprule
    \textbf{Norm} & \multicolumn{1}{l|}{\textbf{Approach}} & \multicolumn{1}{c}{\textbf{A}} & \multicolumn{1}{c}{\textbf{B}} & \multicolumn{1}{c}{\textbf{C}} & \multicolumn{1}{c|}{\textbf{D}} & \multicolumn{1}{c}{\textbf{Avg}} \\
    \midrule
    \multirow{7}[2]{*}{$L_{2}$} & \textbf{PGD} & 39.98  & 39.98  & 39.98  & 39.98  & 39.98  \\
          & \textbf{CW} & 0.04  & 0.04  & 0.04  & 0.04  & 0.04  \\
          & \textbf{Quantize} & 160.19  & 160.22  & 161.76  & 161.78  & 160.99  \\
          & \textbf{Gaussian} & 38.25  & 38.25  & 38.08  & 38.08  & 38.16  \\
          & \textbf{JPEG} & 160.85  & 160.85  & 161.06  & 161.06  & 160.96  \\
          & \textbf{TVM} & 0.52  & 0.52  & 0.51  & 0.51  & 0.51  \\
          & \textbf{Ours} & 4.25  & 4.30  & 4.82  & 5.18  & 4.64  \\
    \midrule
    \multirow{7}[2]{*}{$L_{inf}$} & \textbf{PGD} & 0.03  & 0.03  & 0.03  & 0.03  & 0.03  \\
          & \textbf{CW} & 0.04  & 0.04  & 0.04  & 0.04  & 0.04  \\
          & \textbf{Quantize} & 0.98  & 0.98  & 0.99  & 0.99  & 0.98  \\
          & \textbf{Gaussian} & 0.03  & 0.03  & 0.03  & 0.03  & 0.03  \\
          & \textbf{JPEG} & 0.92  & 0.92  & 0.93  & 0.93  & 0.93  \\
          & \textbf{TVM} & 0.00  & 0.00  & 0.00  & 0.00  & 0.00  \\
          & \textbf{Ours} & 0.04  & 0.04  & 0.04  & 0.02  & 0.04  \\
    \bottomrule
    \end{tabular}%
}
  \label{tab:per}%
\end{table}%

\subsubsection{Quantitative Evaluation}

In this section, we measure the sizes of the generated adversarial examples.
The results are shown in \tabref{tab:per}.
The results show that \tool generates adversarial examples with minimal perturbation sizes. Specifically, \tool generates adversarial examples with the average perturbation size 4.64 for $L_2$ norm and 0.04 for $L_{inf}$ norm.
The results imply \tool generates adversarial examples that are unnoticeable to humans.
Some baselines also generate adversarial examples with imperceptible perturbations, but they cannot affect the NICG model efficiency as expected, making the ``unnoticeable" meaningless.

\subsubsection{Qualitative Evaluation}
\begin{figure}[hbp!]
    \centering
    \includegraphics[width=0.49\textwidth]{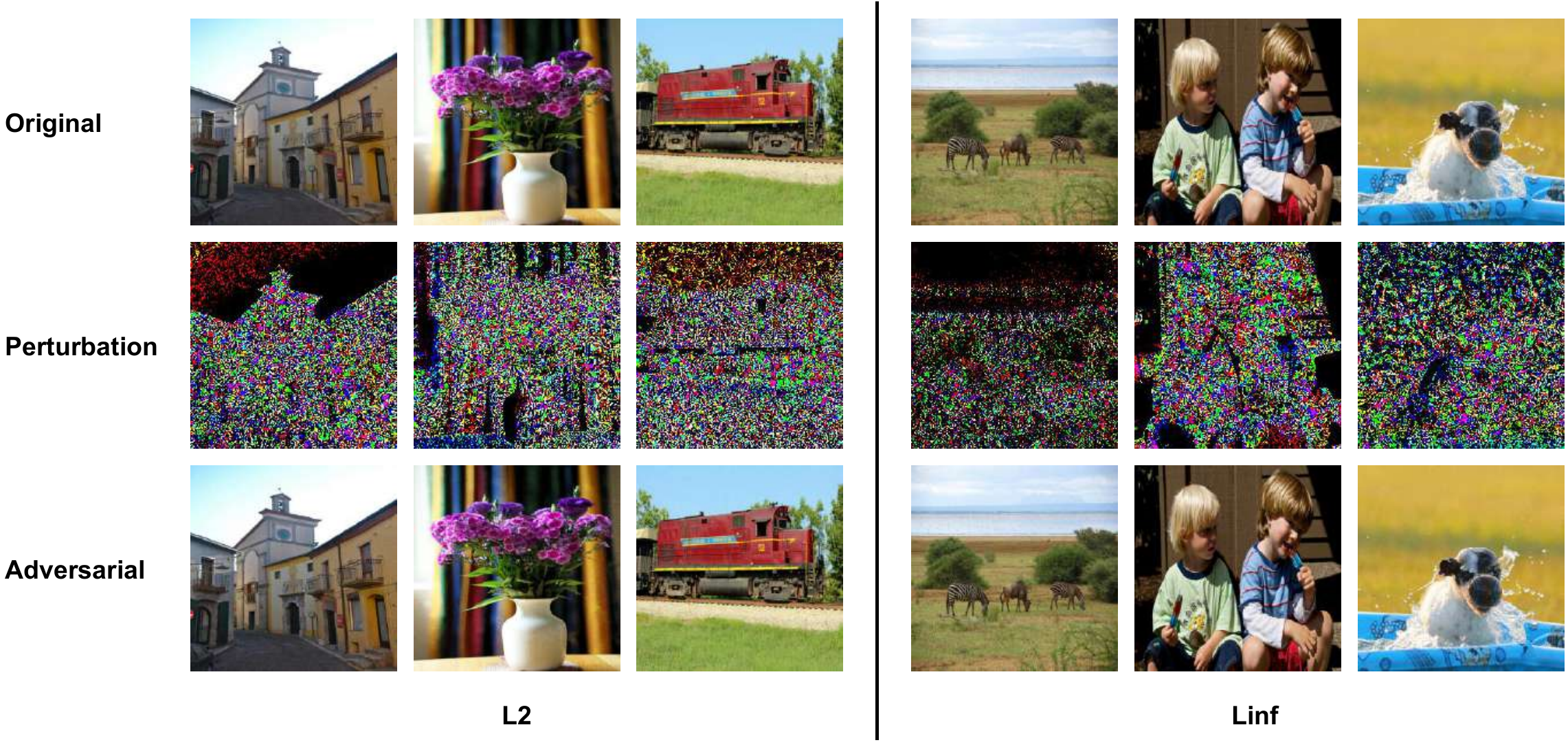}
    \caption{The generated adversarial examples}
    \label{fig:show}
\end{figure}

In this section, we discuss the quality of the generated adversarial inputs based on human perception.
For that purpose, we randomly select six adversarial images and show the selected images in \figref{fig:show} (all generated adversarial images can be found on our website).
The first column shows the benign images, the second column shows the adversarial perturbations used against each benign image, and the third column shows the resultant adversarial images. 
From the results in the first and the third rows, we observe that the added perturbation is not perceptible to humans.

\subsection{More Studies}

\subsubsection{Accuracy \texttt{VS.} Efficiency}

In this section, we evaluate the relationship between accuracy attack and efficiency attack.
The results in \tabref{tab:mainres} show that accuracy-based adversarial examples may not affect NICG model efficiency.
In this section, we evaluate whether efficiency-based adversarial examples will affect NICG model accuracy.
Specifically, we measure the BLEU scores~\cite{papineni2002bleu} of the adversarial examples and the benign examples.
\tabref{tab:acc} shows the BLEU scores of benign examples and adversarial examples generated by \tool. From the results, we can observe that our attack significantly reduces the accuracy of the victim NICG model, decreasing the BLEU scores up to 100\%. This observation indicates that the accuracy-based attack can impact only the NICG model accuracy without reducing efficiency. In contrast, our efficiency-based attack, \tool, can effectively reduce the model efficiency and significantly lower the accuracy.

% Table generated by Excel2LaTeX from sheet 'Sheet5'
\begin{table}[h]
  \centering
  \caption{BLEU scores of benign and adversarial examples}
    \resizebox{.33\textwidth}{!}{
    \begin{tabular}{c|c|c|c|c}
    \toprule
    \multicolumn{2}{c|}{Subjects} & benign & adversarial & decreasae \\
    \hline
    \multirow{2}[2]{*}{A} & L2    & 0.17  & 0.00  & 100.00  \\
          & Linf  & 0.17  & 0.01  & 93.08  \\
    \hline
    \multirow{2}[2]{*}{B} & L2    & 0.20  & 0.00  & 99.02  \\
          & Linf  & 0.20  & 0.02  & 90.94  \\
    \hline
    \multirow{2}[2]{*}{C} & L2    & 0.10  & 0.00  & 98.77  \\
          & Linf  & 0.10  & 0.01  & 90.95  \\
    \hline
    \multirow{2}[2]{*}{D} & L2    & 0.11  & 0.01  & 91.43  \\
          & Linf  & 0.11  & 0.03  & 69.15  \\
    \bottomrule
    \end{tabular}}
  \label{tab:acc}%
\end{table}%

\subsubsection{Hyper-Parameter Sensitively}
In this section, we evaluate the effectiveness of the adversarial examples under different hyper-parameter settings.
Specifically, we set the hyper-parameter $\lambda = [1.0\times10^3, 1.0\times10^4, 1.0\times10^5]$ and run \tool to generate adversarial examples.
From the results in \tabref{tab:sen}, we observe that the adversarial examples generated under different hyper-parameter settings show a stable performance, which implies \tool is not sensitive to hyper-parameter settings.

% Table generated by Excel2LaTeX from sheet 'hyper'
\begin{table}[htbp]
  \centering
  \caption{Effectiveness under different hyper-parameters }
  \resizebox{0.46\textwidth}{!}{
    \begin{tabular}{c|c|c|rrr}
    \toprule
    \multicolumn{1}{l|}{Subject ID} & Norm  & Metric & 10    & 100   & 1000 \\
    \hline
    \multirow{6}[4]{*}{A} &       & I-Loop & 483.86  & 483.86  & 483.86  \\
          & L2    & I-Latency(CPU) & 189.76  & 198.76  & 198.35  \\
          &       & I-Latency(GPU) & 288.43  & 290.40  & 300.32  \\
\cline{2-6}          &       & I-Loop & 360.21  & 354.11  & 344.11  \\
          & Linf  & I-Latency(CPU) & 190.32  & 202.81  & 190.43  \\
          &       & I-Latency(GPU) & 250.32  & 241.90  & 227.32  \\
    \hline
    \multirow{6}[3]{*}{B} &       & I-Loop & 479.32  & 481.32  & 481.32  \\
          & L2    & I-Latency(CPU) & 89.31  & 87.37  & 85.42  \\
          &       & I-Latency(GPU) & 225.43  & 223.38  & 220.43  \\
\cline{2-6}          &       & I-Loop & 283.24  & 271.19  & 271.19  \\
          & Linf  & I-Latency(CPU) & 10.21  & 8.32  & 8.11  \\
          &       & I-Latency(GPU) & 75.43  & 71.77  & 69.31  \\
    \hline
    \multirow{6}[2]{*}{C} &       & I-Loop & 435.56  & 433.58  & 433.58  \\
          & L2    & I-Latency(CPU) & 166.42  & 155.61  & 148.31  \\
          &       & I-Latency(GPU) & 300.32  & 297.37  & 297.32  \\
\cline{2-6}          &       & I-Loop & 388.31  & 379.81  & 370.54  \\
          & Linf  & I-Latency(CPU) & 91.31  & 90.73  & 89.31  \\
          &       & I-Latency(GPU) & 222.32  & 211.41  & 210.32  \\
    \hline
    \multirow{6}[3]{*}{D} &       & I-Loop & 410.23  & 408.90  & 408.90  \\
          & L2    & I-Latency(CPU) & 156.42  & 155.49  & 154.43  \\
          &       & I-Latency(GPU) & 199.32  & 192.58  & 178.31  \\
\cline{2-6}          &       & I-Loop & 117.23  & 115.02  & 113.13  \\
          & Linf  & I-Latency(CPU) & 22.12  & 21.45  & 18.23  \\
          &       & I-Latency(GPU) & 56.43  & 55.36  & 50.13  \\
    \bottomrule
    \end{tabular}%
    }
  \label{tab:sen}%
\end{table}%

\section{Discussion} 

\noindent\textbf{Application.}  
Recently, NICG models have been widely deployed on resource-constrained devices; thus, the need for efficiency robustness evaluation is essential. For example, many mobile applications are developed to help visually impaired persons;  most of those applications rely on the NICG model to provide image explanations to a person. In a situation like crossing a road, the response time should be minimum. Otherwise, fatal accidents can happen. Therefore, the evaluation of efficiency robustness is needed to avoid these scenarios.

\noindent\textbf{Limitation.} 
\tool is a white-box approach, \ie, \tool needs to access the victim NICG model parameters to generate adversarial examples. As we have not evaluated the transferability of the attack, we can not conclude that our attack can also be used in the black-box setting.
However, as \tool is designed for evaluating robustness instead of attacking, the white-box assumption is valid for \tool. 
We leave the black-box evaluation for future work.

\section{Conclusion}

In this paper, our objective is to evaluate the efficiency robustness of NICG models. For this purpose, we propose \tool that generates adversarial efficiency decreasing inputs explores a potential vulnerability of NICG models, \ie, the efficiency of NICG models is inversely proportional to the length of NICG output sequences. Based on the extensive evaluation, we can notice that \tool can generate inputs that significantly decrease NICG models' efficiency. To the best of our knowledge, this is the first adversarial attack exploring the efficiency robustness of NICG models.

\section*{Acknowledgments}

This work was partially supported by Siemens Fellowship and NSF grant CCF-2146443.

{\small
\bibliographystyle{ieee_fullname}
\bibliography{egbib}
}

\clearpage
\section*{Appendix}

\begin{appendices}

\section{More Evaluation Results}
\begin{figure}[h]
    \centering
    \includegraphics[width=0.49\textwidth]{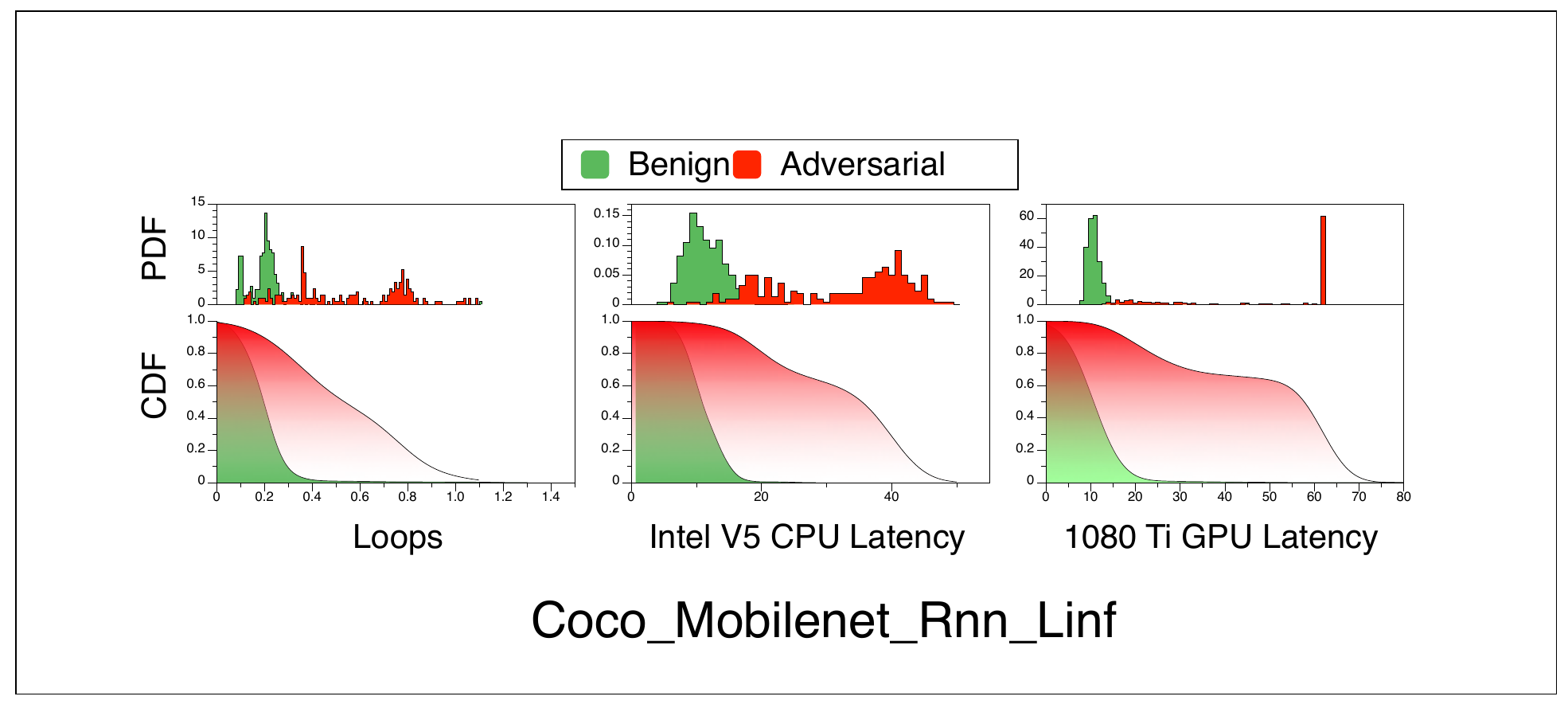}
    \caption{Efficiency Distribution}
    \label{fig:append_1}
\end{figure}
\begin{figure}[h]
    \centering
    \includegraphics[width=0.49\textwidth]{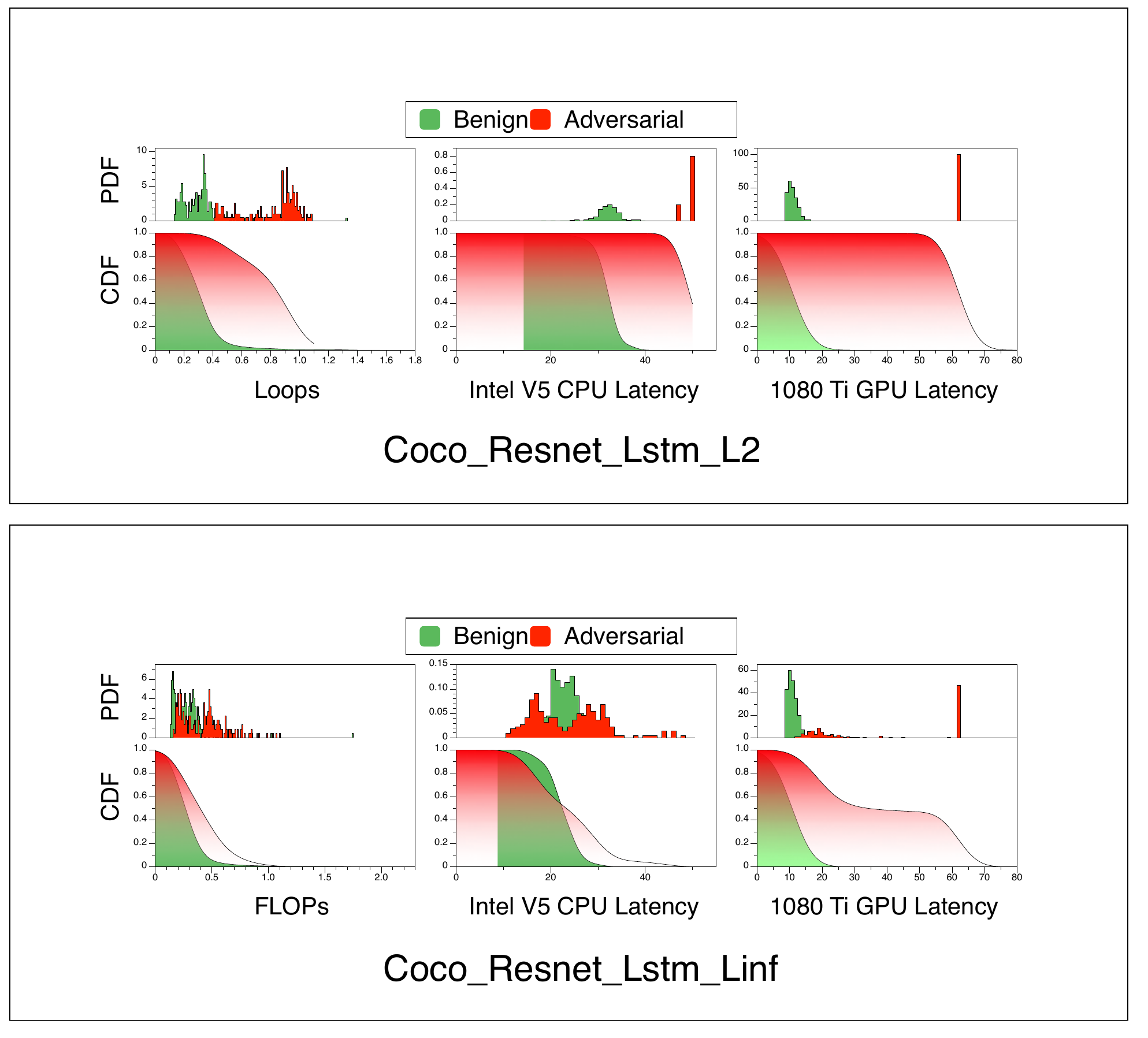}
    \caption{Efficiency Distribution}
    \label{fig:append_2}
\end{figure}
\begin{figure}[h]
    \centering
    \includegraphics[width=0.49\textwidth]{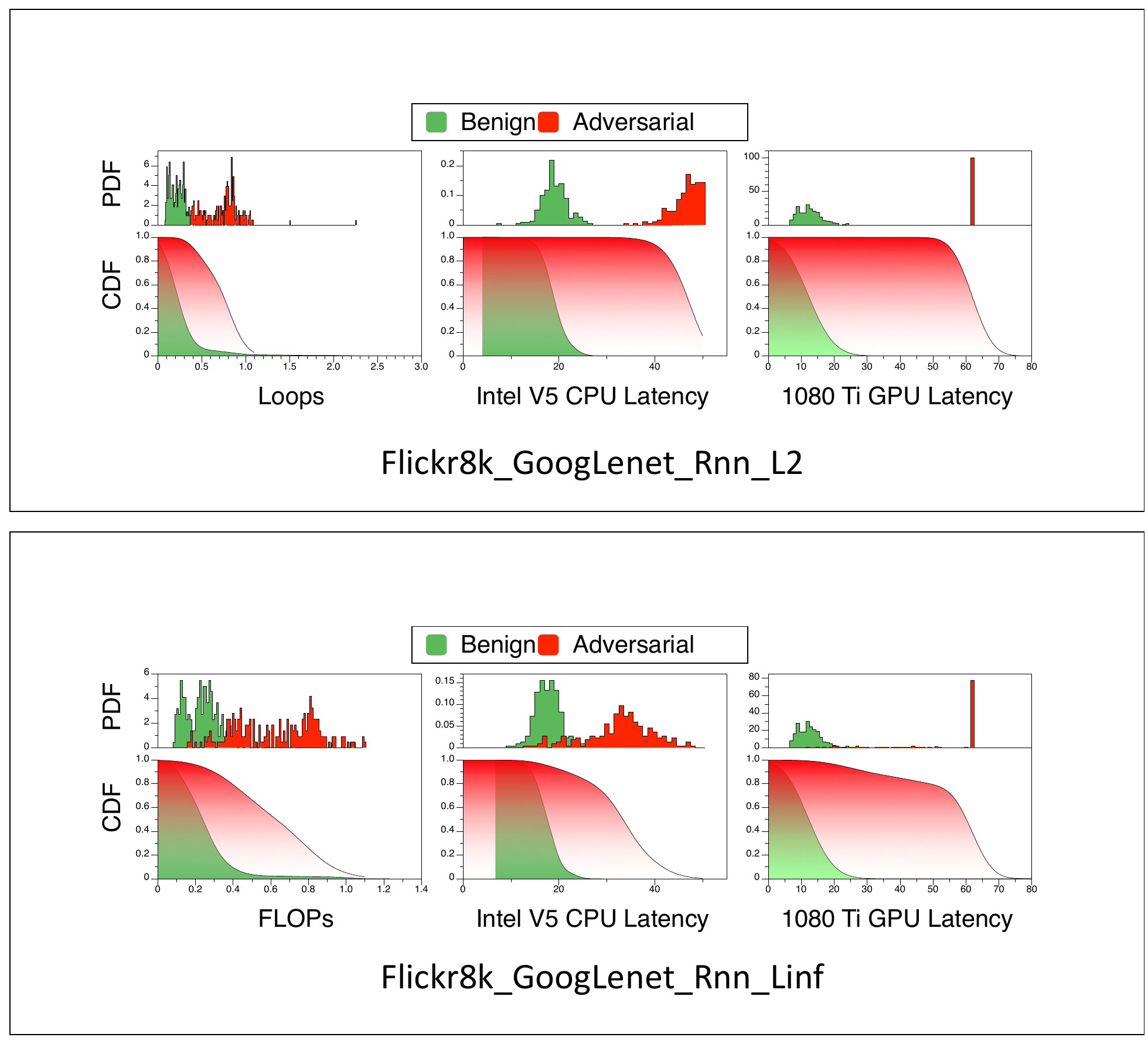}
    \caption{Efficiency Distribution}
    \label{fig:append_3}
\end{figure}
\begin{figure}[h]
    \centering
    \includegraphics[width=0.49\textwidth]{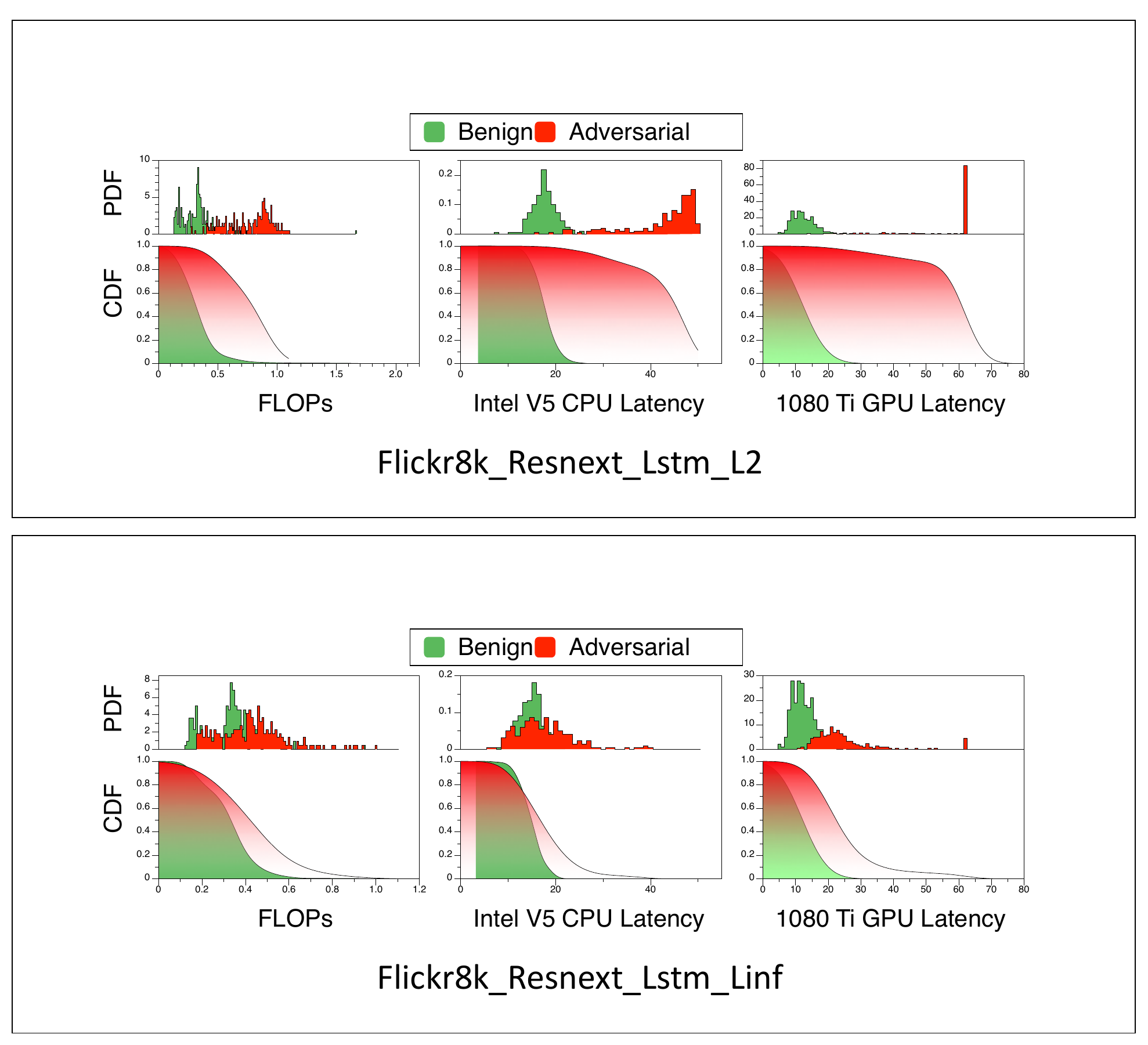}
    \caption{Efficiency Distribution}
    \label{fig:append_4}
\end{figure}

Figure \ref{fig:append_1}, \ref{fig:append_2}, \ref{fig:append_3}, \ref{fig:append_4} show the efficiency distribution of benign images and the generated adversarial images.

The first and second rows represent the Probability Density Function~(PDF) and Cumulative Distribution Function~(CDF) results respectively. 
The area under the CDF curve indicates the efficiency of the NICG model, a larger area indicates the NICG model is less efficiency.
The green area denotes the distribution of benign examples,  and the red areas represent the distribution of adversarial examples generated by \tool.
From the results, we could observe that adversarial examples extremely change the FLOPs and latency distribution of NICG model.
This observation  is consistent with the results in \figref{fig:distribution}.

\section{More Adversarial Examples}

\figref{fig:more_adv} shows more generated adversarial examples, we provide more adversarial examples on the zip files.

\begin{figure*}[ht]
    \centering
    \includegraphics[width=0.95\textwidth]{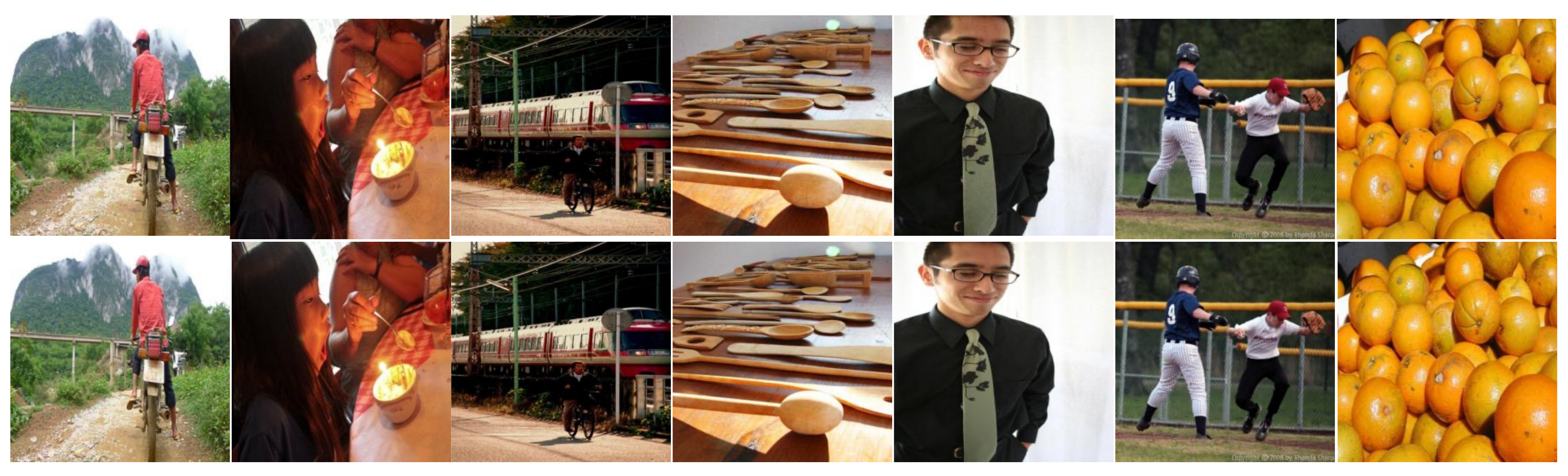}
    \caption{Generated adversarial examples}
    \label{fig:more_adv}
\end{figure*}
\end{appendices}

\end{document}